\documentclass[twocolumn,10pt]{IEEEtran}
\usepackage[utf8]{inputenc}
\usepackage[english]{babel}
\usepackage{amsmath,bm,amsthm,amsfonts,amssymb}
\usepackage{graphicx}
\usepackage{xcolor}
\usepackage{xspace}
\usepackage{subcaption}
\usepackage{mathtools}
\usepackage{cite}
\usepackage{multirow}
\usepackage{bigstrut}
\usepackage{graphicx}
\usepackage{epstopdf}
\usepackage{authblk}
\usepackage{algpseudocode}
\usepackage{algorithm}

\renewcommand{\paragraph}[1]{\vspace{3mm}\noindent\textbf{#1}}

\renewcommand{\b}[1]{{\bm{#1}}}   % bold symbol
\renewcommand{\(}{\left(}           % convenient scalable parenthesis
\renewcommand{\)}{\right)}

\newcommand{\x}{\b{x}}

\renewcommand{\u}{\b{u}}
\renewcommand{\v}{\b{v}}

\renewcommand{\L}{\b{L}}            % The combinatorial Laplacian
\newcommand{\U}{\b{U}}              % an eigenvector matrix
\newcommand{\bLambda}{\b{\Lambda}}  % a diagonal matrix of eigenvalues
\newcommand{\bOmega}{\b{\Omega}}  % a diagonal matrix of eigenvalues
\newcommand{\M}{\b{M}}

\newcommand{\I}{\b{I}}
\newcommand{\X}{\b{X}}
\newcommand{\C}{\b{C}}

\newcommand{\D}{\b{D}}              % singular values
\newcommand{\W}{\b{W}}              % the adjacency matrix
\newcommand{\K}{\b{K}}

\newcommand{\Vs}{\ensuremath{\mathcal{V}}}
\newcommand{\Es}{\ensuremath{\mathcal{E}}}

\newcommand{\tl}{\theta_{\ell}}%
\renewcommand{\ll}{\lambda_{\ell}}%
\newcommand{\kk}{\omega_{k}}%
\newcommand{\eu}{\mathrm{e}}
\newcommand{\ju}{\mathrm{j}}

\newcommand{\LG}{\L_{\hspace{-1px}G}}              % Graph
\newcommand{\LT}{\L_{\hspace{-1px}T}}              % Time

\newcommand{\UG}{\U_{\hspace{-1px}G}}              % Graph
\newcommand{\UT}{\U_{\hspace{-1px}T}}              % Time

\newcommand{\JFT}[1]{\textrm{JFT}\hspace{-.5mm}\(#1\)}

\newcommand{\Rbb}{\mathbb{R}}

\DeclareMathOperator*{\prox}{prox}
\DeclareMathOperator*{\sign}{sgn}
\newcommand{\Hs}[1]{\textrm{H}\hspace{-.5mm}\(#1\)}
\DeclarePairedDelimiter{\norm}{\lVert}{\rVert}
\DeclarePairedDelimiter\abs{\lvert}{\rvert}%

\newcommand{\transpose}{\intercal}                      % the  transpose
\newcommand{\hermitian}{*}                      % the  transpose
\newtheorem{theorem}{Theorem} 
 
 \newtheorem{corollary}{Corollary}

\newcommand{\lrpar}[1]{\left( {#1} \right)}
\newcommand{\lrsquare}[1]{\left[ {#1} \right]}

\graphicspath{ {image/} }

\author[1]{Francesco Grassi}
\author[2]{Nathana\"{e}l Perraudin}
\author[2]{Benjamin Ricaud\vspace*{-1.5ex}}
\affil[1]{Department of Electronics and Telecommunications, Politecnico di Torino, Italy. Email: francesco.grassi@polito.it}
\affil[2]{Signal Processing Laboratory (LTS2), EPFL, Switzerland. Email: firstname.lastname@epfl.ch}

\title{Tracking Time-Vertex Propagation using Dynamic Graph Wavelets}

\begin{document}
\maketitle

\begin{abstract}
Graph Signal Processing generalizes classical signal processing to signal or data indexed by the vertices of a weighted graph. So far, the research efforts have been focused on static graph signals. However numerous applications involve graph signals evolving in time, such as spreading or propagation of waves on a network. The analysis of this type of data requires a new set of methods that fully takes into account the time and graph dimensions. We propose a novel class of wavelet frames named Dynamic Graph Wavelets, whose time-vertex evolution follows a dynamic process. We demonstrate that this set of functions can be combined with sparsity based approaches such as compressive sensing to reveal information on the dynamic processes occurring on a graph. Experiments on real seismological data show the efficiency of the technique, allowing to estimate the epicenter of earthquake events recorded by a seismic network. 
\end{abstract} 
\begin{IEEEkeywords} Graph signal processing, time-vertex signal processing, joint Fourier transform,  dynamic processes on graphs, wave equation \end{IEEEkeywords}

\section{Introduction}
Complex signals and high-dimensional datasets collected from a variety of fields of science, such as physics, engineering, genetics, molecular biology and many others, can be naturally modeled as values on the vertices of weighted graphs \cite{shuman2013emerging, sandryhaila2013discrete}. 
Recently, dynamic activity over networks has been the subject of intense research in order to develop new models to understand and analyze epidemic spreading~\cite{RevModPhys.87.925}, rumor spreading over social networks~\cite{guille2013information,de2013anatomy} or activity on sensor networks. The advances in the graph research has led to new tools to process and analyze time-varying graph and/or signal on the graph, such as multilayer graphs and tensor product of graphs~\cite{kivela2014multilayer,de2013mathematical}. However, there is still a lack of signal processing methods able to retrieve or process information on dynamic phenomena taking place over graphs. For example the wavelets on graphs~\cite{hammond2011wavelets,coifman2006diffusion} or the vertex-frequency transform~\cite{shuman2016vertex} are dedicated to the study of a static signal over a graph. 

Motivated by an increasing amount of applications, we design a new class of wavelet frames named Dynamic Graph Wavelets (DGW) whose time evolution depends on the graph topology and follows a dynamic process. Each atom of the frame is a time-varying function defined on the graph. Combined with sparse recovery methods, such as compressive sensing, this allows for the detection and analysis of time-varying processes on graphs. These processes can be, for example, waves propagating over the nodes of a graph where we need to find the origin and speed of propagation or the existence of multiple sources. 

We demonstrate the efficiency of the DGW on real data by tracking the origin of earthquake events recorded by a network of sensors.

\section{Preliminaries}
\subsection{Notation}
Throughout this contribution, we will use bold upper and lower case letters for linear operators (or matrices) $\M$ and column vectors $\v$, respectively. Furthermore, $\x$ will denote the vectorized version of $\X$. Complex conjugate, transpose and conjugate transpose are denoted as $\overline{\X}$, $\X^\transpose$ and $\X^\hermitian$, respectively. Lower case letters $a$ will denote scalars and upper case letters $A$ will denote fixed constant. For any symmetric positive definite matrix $\M$ with singular value decomposition $\M=\U \bLambda \U^\hermitian$, the matrix function $f(\M)$ is defined as ${f(\M)=\U f(\bLambda) \U^\hermitian}$, where the scalar function $f$ has been applied to each diagonal entry of $\bLambda$. The Kronecker product between two matrices (or vectors) is denoted as $\M_1 \otimes \M_2$, hence, the cartesian product between matrices (or vectors) is ${\M_1 \times \M_2 = \M_1 \otimes \I_2 + \I_1 \otimes \M_2}$, where $\I_n$ is the identity matrix with size equal to $\M_n$.

\subsection{Graph Signal Processing}
Consider a graph $G=(\Vs,\Es,\mathcal{W})$ of $N$ nodes and $E$ edges, where $\Vs$ indicates the set of nodes and $\Es$ the set of edges. The weight function $\mathcal{W}:\Vs\times\Vs \rightarrow \Rbb$ reflects to what extent two nodes are related to each other. $\W_G$ is the weight matrix associated to this function. The combinatorial Laplacian  ${\LG = \D_G -\W_G}$ associated to the graph $G$ is always symmetric positive semi-definite, therefore, due to the spectral theorem, it is characterized by a complete set of orthonormal eigenvectors~\cite{chung1997spectral}. We denote them by $\UG(n,\ell)=\u_{\ell}(n)$. The Laplacian matrix can thus be decomposed as $\LG= \UG\bLambda_G \UG^\hermitian$, with $\bLambda_G(\ell,\ell)=\ll$.   Let $\x: \Vs \rightarrow \Rbb$ be a graph signal defined on the graph nodes, whose $n$-th component $\x(n) \in \Rbb$ represents the value of signal at the $n$-th node.  The Graph Fourier Transform (GFT) of $\x$ is $\widetilde{\x}=\UG^\hermitian\x$ and its inverse $\x= \UG\widetilde{\x}$.

\section{Time-vertex representation}
\subsection{Definition}
Let $\X\in\Rbb^{N \times T}$ be a set of $N$ temporal signals of length $T$. The signals are evolving with time over the $N$ vertices of the graph $G$. We call the cartesian product between time and graph domain the time-vertex domain and $\X$ time-vertex signal. The time-vertex domain can be interpreted as a cartesian product between the generic graph ${G=(\Vs,\Es,\mathcal{W}})$ with Laplacian $\LG$ and the ring graph $G_T$  (assuming periodic boundary conditions in time) with Laplacian $\LT$. The joint Laplacian is
\begin{equation}\label{eq:jlap}
\lrpar{\LT\times \LG}\x=\LG \X + \X\LT
\end{equation}
where the second term is obtained using the property of the Kronecker product $\lrpar{\M_1\otimes\M_2}\x = \M_2\X\M_1^{\transpose}$.

In equation~\eqref{eq:jlap} the Laplacian $\LT$ represents the discrete second order derivative with respect to time:
\[
[\X \LT](n,t) = \X(n,t+1)-2\X(n,t)+\X(n,t-1).
\]
It can be decomposed as  ${\LT=\UT \bOmega \UT^{\hermitian}}$ where $\UT$ is the discrete Fourier basis~\cite{strang1999discrete} and ${\bOmega(k,k)=\omega_k}$ are the eigenvalues of the classical DFT that are linked to the normalized discrete frequencies $\frac{k}{T}$ by the following relation: 
\begin{equation}\label{eq:freq}
\omega_k = 2\left(\cos\left(\pi\frac{k}{T}\right)-1\right).
\end{equation}

\subsection{Joint Time-Vertex Fourier Transform}
Since the time-vertex representation is obtained from the cartesian product of the two original domains, the joint time-vertex Fourier transform (JFT) is obtained by applying the GFT on the graph dimension and the DFT along the time dimension~\cite{loukas2016frequency}:
\begin{equation*}
\widehat{\X}(\ell,k) = \dfrac{1}{\sqrt{T}}\sum_{n=1}^N\sum_{t=0}^{T-1}\X(n,t)\u_{\ell}^*(n)\eu^{-\ju \omega_k\frac{t}{T}}
\end{equation*}
that can be conveniently rewritten in matrix form as:
\begin{equation}\label{eq:jft}
\widehat{\X} =\JFT{\X} = \UG^{\hermitian}\X\overline{\U}_T.
\end{equation}
The spectral domain helps in defining the localization of functions on the graph, as in~\cite{hammond2011wavelets}.

%%%%%%%%%%%%%%%%%%%%%%%%%%%%%%%%
\section{Dynamic Graph Wavelets}

The DGW differ from classical wavelets as they are not dilated versions of an initial mother wavelet. Indeed, they are propagating functions on the graph that evolve in time, according to a PDE. We will use the joint representation for the signal to characterize spectral relationships between the two domains and solve the PDE in the spectral domain obtaining an useful tool to analyze time-vertex signals that evolve according to this dynamic process. Finally, we will use the kernel to build the set of DGW. Because of the lack of translation invariance of graph, the kernel will always act in graph spectral domain and will be localized on the graph using the localization operator as in~\cite{hammond2011wavelets}. On the contrary the time dependence can be defined either in the spectral or in the time domain.  In this contribution, we use for convenience the time domain. 

%%%%%%%%%%%%%%%%%%%%%%%%%%%%%%%
\subsection{Heat diffusion on graph}
Let us first provide a basic example for our model. The diffusion of heat on a graph can be seen as a simple dynamic process over a network. It is described by the following (discretized) differential equation:
\begin{equation}
\X(n,t)-\X(n,t-1) = -s  \LG \X, 
\end{equation}
with initial distribution ${\X(i,0)=\b{\psi}(i)}$. The closed form solution is given by $\X = \eu^{-s t \LG}\b{\psi} $. Therefore, the heat diffusion spectral kernel is
\begin{equation}\label{eq:heat}
\widetilde{K}(s\ll,t)=\eu^{-s \ll t}
\end{equation} 
where the parameter $s$ is the thermal diffusivity in classic heat diffusion problems and can be interpreted as a scale parameter for multiscale dynamic graph wavelet analysis~\cite{coifman2006diffusion}. This equation models the spreading of a function on the graph over time. However, in the present work we want to focus on a propagating process, moving away from an initial point as time passes. Hence we introduce a second model.

\subsection{The wave equation on graphs}
To model functions evolving on a graph, we will use mainly the PDE associated to the wave equation. Here, the wave equation is defined on the graph, and, as such, differs from the standard one. This partial differential equation relates the second order derivative in time to the spatial Laplacian operator of a function:
\begin{align} 
(\LT\otimes\I_G)\, \x&=-\alpha^2 ( \I_T\otimes\LG ) \, \x \nonumber \\
\X \LT&=-\alpha^2 \LG \, \X \label{eq:wavepde}
\end{align}
where $\alpha$ is the propagation speed parameter. Assuming a vanishing initial velocity, i.e. first derivative in time of the initial distribution equals zero, the solution to this PDE can be written using functional calculus as~\cite{durran2013numerical}:
\begin{equation}\label{eq:wave_sol}
\X(\,\cdot\,,t)=K(s\LG,t)\b{\psi}=\K_{t,s}\b{\psi}
\end{equation}
where ${\b{\psi}(n)=\X(n,0)}$ and ${\K_{t,s}=K(s\LG,t)}$ is the matrix function $K$ applied to the scaled Laplacian $s\LG$ and parametrized by the time $t$.
Notice that we use the scale ${s=\alpha^2}$ to represent the speed parameter of the propagation. Substituting \eqref{eq:wave_sol} into \eqref{eq:wavepde}, we obtain
\begin{equation}\label{eq:wave_sub}
\K_{t,s}\b{\psi}\L_T=-s \LG \K_{t,s}\b{\psi}.
\end{equation}
To obtain a closed form solution for the kernel $K_s$ we analyze the equation~\eqref{eq:wave_sub} in the graph spectral domain:
\begin{equation} \label{eq:condpde}
\widetilde{\K}_{t,s}\widetilde{\b{\psi}} \L_T  = -s\b{\Lambda}_G \widetilde{\K}_{t,s}\widetilde{\b{\psi}}
\end{equation}
where ${\widetilde{\K}_{t,s}=K(s\b{\Lambda}_G,t)}$. Equation~\eqref{eq:condpde} requires $K(s\ll,t)$ to be an eigenvector of $\L_T$. From~\eqref{eq:freq} we obtain:
\begin{equation}\label{eq:wave_ker}
\widetilde{K}(s\ll,t) =  \cos\lrpar{t \arccos\lrpar{1- \frac{s\ll}{2}} }.
\end{equation}
Since the $\arccos(x)$ is defined only for ${x\in\lrsquare{-1,1}}$, to guarantee filter stability the parameter $s$ must satisfy ${s<4/\lambda_{max}}$. This result is in agreement with stability analysis of numerical solver for discrete wave equation \cite{Duncan1998numerical}. 

The wave equation is a hyperbolic differential equation and several difficulties arise when discretizing it for numerical computation of the solution\cite{durran2013numerical}. Moreover, the graph being an irregular domain, the solution of the above equation is not  any more a smooth wave after few iterations. Here we focus on the propagation (away from its origin) of the wave rather than its exact expression.

%%%%%%%%%%%%%%%%%%%%%%%%%%%%%%%
\subsection{General definition}
In the following, we will generalize the DGW using arbitrary time-vertex kernel. The goal is to design a transform that helps detecting a class of dynamic events on graphs. These events are assumed to start from an initial joint distribution $\b{\Phi}_{m,\tau}(n,t) = [\b{\psi}_{m}\otimes\b{\phi}_{\tau}^{\transpose}](n,t)$ localized around vertex $m$ and time $\tau$. The general expression of the DGW $W_{m,\tau,s}$ at time $t$ and at vertex $n$ can be written as:
\begin{equation}\label{eq:dgw}
W_{m,\tau,s}(n,t)=\left[\K_{t,s}\b{\Phi}_{m,\tau}\right](n,t), \hspace{0.1cm} \forall\, t\geq0
\end{equation}
where, like earlier, ${\K_{t,s}=K(s\LG,t)}$ is the matrix function $K$ applied to the scaled Laplacian $s\LG$ and parametrized by the time $t$. Depending on the dynamic graph kernel, the DGW can resemble a wave solution of the wave equation, a diffusion process, or a generic dynamic process.

\subsection{Causal Damped Wave Dynamic Graph Kernel}
We define the DGW to be the solutions of Eq.~\eqref{eq:wavepde}, for different $\alpha=\sqrt{s}$. In addition, we require two other properties. Firstly, we want the wave to be causal, i.e. to have an initial starting point in time.
Secondly, in many applications, the wave propagation is affected by an attenuation over time. We thus introduce a damping term. The DGW defined in the graph spectral domain is thus
\begin{equation}\label{eq:dw_dgw}
\widetilde{W}_{s}(\ll,t)=\Hs{t}\eu^{-\beta t}\cos\lrpar{t \arccos\lrpar{1- \frac{s\ll}{2}} },
\end{equation}
where $\Hs{t}$ is the Heaviside function and $\eu^{-\beta t}$ is the damped decaying exponential function in time. 

The damping term has two remarkable effects. Firstly, it lower the importance of the chosen boundary conditions in time (e.g. periodic or reflective) as the wave vanishes before touching them. Secondly, it favors the construction of a frame of DGW: we will see in the following that $\beta$ is involved in the lower frame bound of the DGW.

\subsection{Dynamic Graph Frames}
We define $S_W$ as the DGW analysis operator. The wavelet coefficients $C$ are given by
\begin{align*}
\C(m,\tau,s)&=\left\{ S_{W}(\X)\right\} (m,\tau,s) = \sum_{n,t}W_{m,\tau,s}(n,t)\X(n,t)\\
  &= \frac{1}{\sqrt{T}}\sum_{\ell,k}\widehat{W}_{s}(\ll,\omega_k)\widehat{\X}(\ell,k)\u_{\ell}(m)\eu^{-\ju \omega_k\frac{\tau}{T}},
\end{align*}
and the synthesis operator gives
\begin{align*}
\X'(n,t) &= \left\{ S_{W}^{\transpose}(\b{C})\right\} (n,t) = \sum_{j,\tau,s}W_{m,\tau,s}(n,t)\C(m,\tau,s)\\
&= \frac{1}{\sqrt{T}} \sum_{s}\sum_{\ell,k}\widehat{W}_{s}(\ll,\omega_k)\widehat{\C}(\ell,k,s)\u_{\ell}(i)\eu^{-2\pi \ju \omega_k\frac{\tau}{T}}.
\end{align*}
The following theorem provides conditions to assert that no information will be lost when these operators are applied to a time-vertex signals. This implies that any signal $\X$ can be constructed from the synthesis operation: $ \X = S_{W}^{\transpose}(\C)$.
\begin{theorem}\label{theo:frame}
If the set of time-vertex DGW satisfies:
\begin{align*}
A&=\min_{l,k}\sum_{s}\abs{\widehat{W}_{s}(\ll,\omega_{k})}^{2}>0\\
B&=\max_{l,k}\sum_{s}\abs{\widehat{W}_{s}(\ll,\omega_{k})}^{2}<\infty \qedhere
\end{align*} 
with $0 < A \leq B < \infty$, then $S_W$ is a frame operator in the sense:
\begin{equation}\label{eq:framebound}
A\norm{\X}_2^2\leq \norm*{\left\{ S_{W}(\X)\right\}}^2_2 \leq B\norm{\X}_2^2
\end{equation}
for any time-vertex signal $\X$ with $\norm{\X}_2>0$. 
\end{theorem}
\begin{proof}
In the joint spectral domain we can write:
\begin{align*}
&\norm*{\left\{ S_{W}(\X)\right\}}^2_2 = \sum_{m, \tau,s} \abs*{\left\{ S_{W}(\X)\right\} (m,\tau,s)}^2\\
&=\sum_{m, \tau, s}\abs*{\sum_{n, t} \X(n,t) \sum_{\ell, k} \widehat{W}_{s}(\ll,\kk) \u_{\ell}^{*}(n)\u_{\ell}(m)\eu^{-\ju\kk\frac{t-\tau}{T}}}^2 \\
&=\sum_{s, m, \tau}\lrpar{\sum_{n, t} \X(n,t) \sum_{\ell, k} \widehat{W}_{s}(\ll,\kk) \u_{\ell}^{\hermitian}(n)\u_{\ell}(m)\eu^{-\ju\kk\frac{t-\tau}{T}}}\\
&\lrpar{\sum_{n',t'} \X(n',t') \sum_{\ell', k'} \hat{W}_{s}(\lambda_{\ell'},\omega_{k'}) \u_{\ell'}^{\hermitian}(n')\u_{\ell'}(m)\eu^{-\ju\omega_{k'}\frac{t'-\tau}{T}}}^{\hermitian}\\
&=\sum_{s,\ell, k} \widehat{W}_{s}(\ll,\kk) \widehat{W}_{s}^{*}(\ll,\kk) \widehat{\X}(\ell, k)\widehat{\X}^{\hermitian}(\ell, k)\\
&=\sum_{s,\ell, k} \abs{\widehat{W}_{s}(\ll,\kk)}^{2}\abs{\widehat{\X}(\ell, k)}^{2}=\sum_{s}\norm{\widehat{W}_{s}\cdot \widehat{\X}}_2^2.
\end{align*}

Using Parseval relation $\|\hat{\X}\|_2=\|\X\|_2$, we find
$$
A\|\X\|_2^2 = A\|\widehat{\X}\|_2^2 \leq \norm*{\left\{ S_{W}(\X)\right\}}^2_2 \leq B\|\widehat{\X}\|_2^2 = B\|\X\|_2^2
$$
where $\norm{\cdot}_2$ is used here for the Froebenius norm, i.e: ${\norm{\X}_2 =\norm{\x}_2}$.
\end{proof}

In the following we will use this condition to prove that the DGW given in equation~\eqref{eq:dw_dgw} is a frame.
\begin{corollary}
The set of DGW defined by Eq.\eqref{eq:dw_dgw} is a frame for all $\beta>0$.
\end{corollary}
\begin{proof}
We define $\tl = \arccos(1-\frac{s\ll}{2})$. The DGW in the joint spectral domain is
\begin{align*}
&\abs*{\widehat{W}(\ll,\kk)}^2	=	\abs*{\sum_{t>0}\eu^{-\beta t}\cos(t\arccos(1-\frac{s\ll}{2}))\eu^{-\ju \kk t}}^2\\
%&=\frac{1}{2}\sum_{t}\eu^{-\beta t}(\eu^{\ju t\arccos(1-\frac{s\ll}{2})}+\eu^{-\ju t\arccos(1-\frac{s\ll}{2})})\eu^{-\ju \kk t}\\
%&=	\frac{1}{2}\sum_{t}\eu^{-\beta t}\eu^{-\ju (\kk + \arccos(1-\frac{s\ll}{2}))t} \\& \quad +\frac{1}{2}\sum_{t}\eu^{-\beta t}\eu^{-\ju (\kk- \arccos(1-\frac{s\ll}{2}))t}\\
&=\abs*{\frac{1}{2}\lrpar{\dfrac{1}{1-\eu^{-(\beta + \ju(\kk+\tl)}}+ \dfrac{1}{1-\eu^{-(\beta +\ju(\kk-\tl)}}}}^2\\
&=\abs*{\frac{1}{2} \dfrac{2 - \eu^{-\beta- \ju\kk}(\eu^{-\ju\theta}+\eu^{\ju\theta })}{1 - \eu^{-\beta- \ju\kk}(\eu^{-\ju\theta}+\eu^{\ju\theta}) + \eu^{ -2\beta -2\ju \kk } }}^2\\
 &= \abs*{\dfrac{1 - \eu^{-\beta- \ju\kk}\cos\tl}{1 - 2\eu^{-\beta- \ju\kk}\cos{\tl} + \eu^{ -2\beta -2\ju \kk } }}^2\geq \abs*{\dfrac{1-\eu^{-\beta}}{4}}^2>0.
\end{align*}
Hence $A>0$. 
To prove that $B<\infty$ we find the roots of the denominator of the above expression. We call $z=\eu^{-\beta- \ju\kk}$ and we obtain the following equation:
\begin{align*}
&1 - 2z\cos{\tl} + z^2=0
\end{align*}
whose roots are $\abs{\dot{z}}=\abs{\cos(\tl)\pm \ju\sin(\tl)}=1$. Since ${\abs{z}=\eu^{-\beta}}$, $\abs{z}\neq\abs{\dot{z}}\ \forall \beta>0$.
\end{proof}

\begin{algorithm}[b]
\caption{FISTA for Problem~\eqref{eq:optim}}
\label{CHalgorithm}
\begin{algorithmic}
\State INPUT: $c_1 = y$, $u_0 = y$, $t_1 = 1$, $\epsilon > 0$
\For{ $j = 1,\dots J$ }
\State $u_{j+1} = \prox_{\nu_jh}(c_{j}-\nu_j\nabla g(c_j))$
\State $t_{j+1} = \frac{1+\sqrt{1+4t_j^2}}{2}$
\State $c_{j+1} = u_j +\frac{t_j-1}{t_{j+1}} (u_j-u_{j-1})$
\If{$\frac{\|c_{j+1} - c_{j}\|_2^2}{\| c_{j}\|_2^2+\delta}<\epsilon$}
\State BREAK
\EndIf
\EndFor
%\EndProcedure
\end{algorithmic}
\end{algorithm}
\section{Sparse representation}
Particular processes, such as wave propagation, can be well approximated by only a few elements of the DGW, i.e. the DGW transform of the signal is a sparse representation of the information it contains. In that case, we inspire ourselves from compressive sensing techniques and define the following convex minimization problem
\begin{equation}\label{eq:optim}
\dot{\C} = \arg\min_{\C}\|S_W^{\transpose}(\C)-\b{Y}\|_{2}^{2}+\gamma\|\C\|_{1}.
\end{equation}
Here $\gamma$ is the parameter controlling the trade-off between the fidelity term $\|S^{\transpose}_W\C-\b{Y}\|_{2}^{2}$  and the sparsity assumption of the DGW coefficients ${\|\C\|_{1}=\sum_{m,\tau,s}\abs{\C(m,\tau,s)}}$. The solution $\dot{\C}$ provides useful information about the signal. Firstly, the synthesis $S^{\transpose}_W\dot{\C}$ is a de-noised version of the original process. Secondly, from the position of the non zero coefficients of $\dot{\C}$, we can derive the origin on the graph $m$ and in time $\tau$, the speed of propagation $s$ and the amplitude $|\C(m,\tau,s)|$ of the different waves. 

Problem~\eqref{eq:optim} can be solved using proximal splitting methods~\cite{combettes2011proximal} and the fast iterative soft thresholding algorithm (FISTA)~\cite{beck2009fast} is particularly well suited. 
Let us define $g(\C)=\|S^{T}_W\C-\b{Y}\|_{2}^{2}$, the gradient of $g$ is ${\nabla g(\C) =  2 S_W(S^{T}_W \C -\b{Y}).}$  
Note that the Lipschitz constant of $\nabla_g$ is $2D$. We define the function $h(\C) = \gamma \|\C\|_1$. The proximal operator of $h$ is the $\ell_1$ soft-thresholding given by the elementwise operations (here $\circ$ is the Hadamard product) 
\[
\prox_{\gamma, h }(\C) = \C + \sign(\C) \circ \max (|\C|-\gamma ,0).
\]
The FISTA algorithm \cite{beck2009fast} can now be stated as Algorithm \ref{CHalgorithm},
where $\nu$ is the step size (we use $\nu = \frac{1}{2D}$), $\epsilon$ the stopping tolerance and $J$ the maximum number of iterations. $\delta$ is a very small number to avoid a possible division by $0$. %
Our implementation of the frame $S_W$ is based on the GSPBox~\cite{perraudin2014gspbox} and Problem \eqref{eq:optim} is solved using the UNLocBoX~\cite{perraudin2014unlocbox}.

\begin{table}
 \normalsize
\centering
\begin{tabular}{c c c c c c}
& \multicolumn{5}{c}{SNR [dB]}\\ \cline{2-6}
\bigstrut[t]
Event ID & 100 & 20 & 10 & 2 & 0 \\
\hline
\bigstrut[t]
2014p139747 & 28.88 &  28.92 &  29.24 &  31.41 &  32.35\\
2015p822263 & 28.35 &  28.25 &  29.13 &  28.53 &  29.97\\
2015p850906 & 17.80 &  18.38 &  15.18 &  16.60 &  21.26\\
2016p235495 & 37.39 &  37.41 &  37.58 &  37.83 &  37.83\\
\hline
\end{tabular}
\caption{Distance in kilometers between real and estimated epicenter for different seismic events and decreasing SNR.}
\label{tab:error}
\end{table}
\section{Application}
\subsection{Earthquake epicenter estimation}
We demonstrate the performance of the DGW on a source localization problem, where a dynamical event evolves according to a specific time-space behavior. We analyze waveforms recorded by seismic stations geographically distributed in New Zealand, connected to the GeoNet Network. The graph is constructed using the coordinates of the available seismic stations and connecting the closest nodes. We consider different seismic events whose epicenters were located in different areas of New Zealand\footnote{The dataset is freely available at http://www.geonet.org.nz/quakes}. Each waveform consists of $300$ seconds sampled at $100\,$Hz, starting few seconds before the seismic event. Seismic waveforms can be modeled as oscillating damped waves. This model is valid when the spatial domain where the waves are propagating is a continuous domain or a regular lattice~\cite{lowrie2007fundamentals}. Here, the domain is the network of sensors and we assume that a damped wave propagating on this network is still a good approximation. Thus we expect the waveforms of the DGW defined in Eq.\eqref{eq:dw_dgw} to be good approximations of the seismic waves recorded by the sensors. We create a frame of DGW, $S_W$, using 10 different values for the propagation velocity parameter $s$ linearly spaced between 0 and 2 (corresponding to physically plausible values). The damping $\beta$ was fixed and chosen to fit the damping present in the seismic signals. 

To estimate the epicenter of the seismic event we solved the convex optimization problem \eqref{eq:optim}. The sparse matrix $\C$ contains few non-zero coefficients corresponding to the waveforms that constitute the seismic wave. We averaged the coordinates of the vertices corresponding to the sources of the waves with highest energy coefficients. Figure~\ref{fig:seismic} shows the results of the analysis for different seismic events. For each plot, the recorded waveforms are shown superposed using different colors. Real and estimated epicenters are shown respectively with a red square and a black circle on the graph plots. 

Finally, we investigated the performance of the source localization algorithm by adding white Gaussian noise to the signals, decreasing the SNR of the waveforms from $100$ to $0$\,dB, such that the SNR is the same for all the waveforms. Table~\ref{tab:error} shows the distance between the real and estimated epicenter in kilometers in four different events and increasing amounts of noise. The small variations of the results demonstrate the high robustness of the method.

\begin{figure*}
\centering
\includegraphics[width=\textwidth]{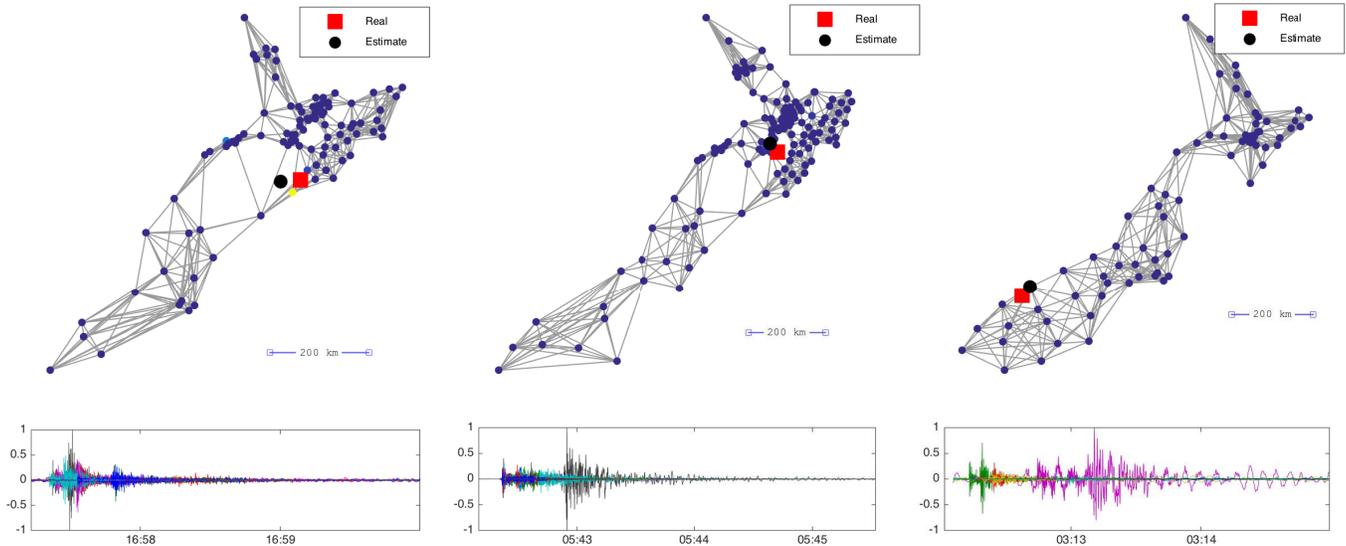}
\caption{Results for 3 different seismic events in New Zealand. Top: the graph is created using the coordinates of the available stations for each event and connecting the closest stations. The red squares and the black dots are the true and estimated sources of the seismic wave respectively. Bottom: Signal recorded by the sensors (different color per sensor) over time for each event.}\label{fig:seismic}
\end{figure*}

\bibliographystyle{IEEEbib}
\bibliography{bibliography}

\end{document}